\documentclass[letterpaper]{article} 
\usepackage{aaai24}  
\usepackage{times}  
\usepackage{helvet}  
\usepackage{courier}  
\usepackage[hyphens]{url}  
\usepackage{graphicx} 
\usepackage{natbib}  
\usepackage{caption} 
\usepackage[ruled,linesnumbered, noend]{algorithm2e}
\SetCommentSty{emph}
\usepackage[table,xcdraw]{xcolor}
\SetAlgoSkip{}
\usepackage{newfloat}
\usepackage{listings}
\DeclareCaptionStyle{ruled}{labelfont=normalfont,labelsep=colon,strut=off} 
\usepackage{color} 
\usepackage{amsfonts}
\usepackage{amsmath} 

\usepackage{amsthm}
\newtheorem{lem}{Lemma}

\newtheorem{thm}{Theorem}

\newcommand{\confspace}[2]{\mathcal{Q}^{#1}_{#2}}
\newcommand{\compconf}[1]{q_{#1}} 
\newcommand{\conf}[2]{q^{#1}_{#2}} 
\newcommand{\confstart}[1]{q^{#1}_{\text{start}}}
\newcommand{\confgoal}[1]{q^{#1}_{\text{goal}}}
\newcommand{\robot}[1]{\mathcal{R}_{#1}}
\newcommand{\open}{\text{OPEN}}
\newcommand{\focal}{\text{FOCAL}}

\makeatletter
\renewenvironment{proof}[1][\proofname]{\par
  \vspace{-\topsep}
  \pushQED{\qed}%
  \normalfont
  \topsep3pt \partopsep3pt 
  \trivlist
  \item[\hskip\labelsep
        \itshape
    #1\@addpunct{.}]\ignorespaces
}{%
  \popQED\endtrivlist\@endpefalse
  \addvspace{3pt plus 3pt} 
}
\makeatother

\usepackage[switch, modulo]{lineno}

\title{Accelerating Search-Based Planning for Multi-Robot Manipulation by Leveraging Online-Generated Experiences}
\author {
    Yorai Shaoul\equalcontrib,
    Itamar Mishani\equalcontrib,
    Maxim Likhachev,
    Jiaoyang Li
}
\affiliations {
    Carnegie Mellon University\\
    \{yshaoul, imishani, maxim\}@cs.cmu.edu, 
    jiaoyangli@cmu.edu
}

\begin{document}
\maketitle

\begin{abstract}
An exciting frontier in robotic manipulation is the use of multiple arms at once. 
However, planning concurrent motions is a challenging task using current methods. The high-dimensional composite state space renders many well-known motion planning algorithms intractable.
Recently, Multi-Agent Path-Finding (MAPF) algorithms have shown promise in discrete 2D domains, providing rigorous guarantees. However, widely used conflict-based methods in MAPF assume an efficient single-agent motion planner. This poses challenges in adapting them to manipulation cases where this assumption does not hold, due to the high dimensionality of configuration spaces and the computational bottlenecks associated with collision checking.
To this end, we propose an approach for accelerating conflict-based search algorithms by leveraging their repetitive and incremental nature -- making them tractable for use in complex scenarios involving multi-arm coordination in obstacle-laden environments. 
We show that our method preserves completeness and bounded sub-optimality guarantees, and demonstrate its practical efficacy through a set of experiments with up to 10 robotic arms.

\end{abstract}

\section{Introduction}

The synchronous use of multiple robotic arms may enable new application domains in robotics and enhance the efficiency of tasks traditionally carried out by a single arm. For example, in pick-and-place tasks (e.g., Fig. \ref{fig:teaser}), multiple arms can potentially be more efficient than a single one, and in a manufacturing setting, multiple arms can be used to assemble a product collaboratively, unlocking the capability to perform tasks that are beyond the scope of a single arm.
However, the inherent complexity of single-agent motion planning for robot manipulation \cite{canny1988complexity} makes it challenging to plan for multiple arms while ensuring collision-free paths, and thus has left Multi-Robot-Arm Motion Planning (M-RAMP) a relatively under-explored frontier in robotics.

To enable the use of multiple arms in more complex scenarios, we propose a method for accelerating M-RAMP. 
Our approach capitalizes on a key observation: widely-used Multi-Agent Path Finding (MAPF) algorithms exhibit a significant degree of repetitive planning. We exploit this repetitiveness by developing a method that leverages experiences gathered during the planning process. Unlike previous approaches that utilize incremental search techniques \cite{boyarski2021iterative}, we allow the use of bounded sub-optimal search techniques, which are crucial for exploring high-dimensional state spaces. To this end, we accelerate the single-agent planning process by reusing online-generated path experiences to speed up multi-agent search, ensuring both completeness and solution quality guarantees.

\begin{figure}[t]
    \centering
    \includegraphics[width=0.99\linewidth]{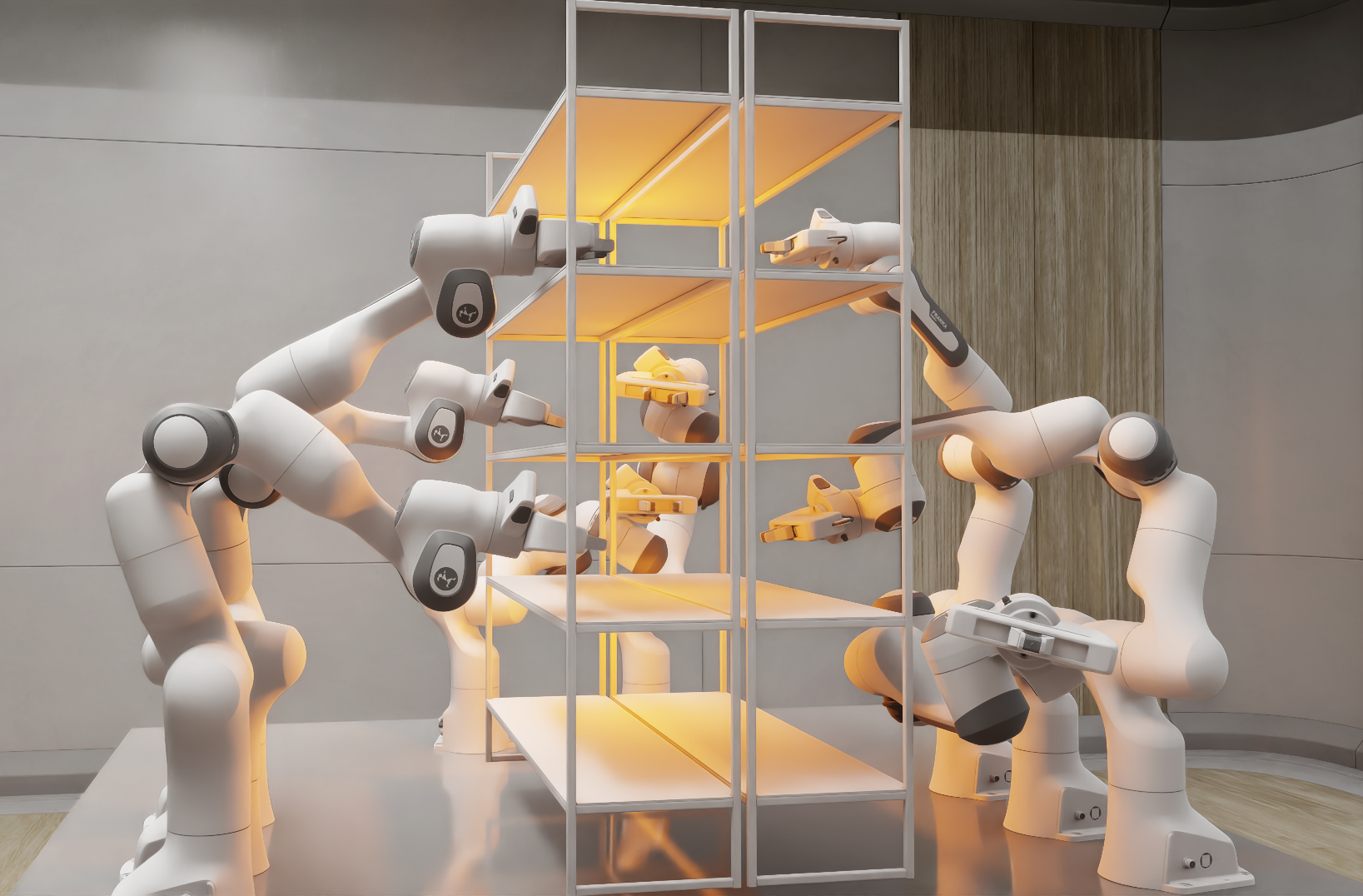}
    \caption{Eight robotic manipulators, each of 7-DoF, collaborating in a shelf-rearrangement pick and place task. Planning concurrent motions for all arms requires a motion planner capable of efficiently exploring a single arm's high-dimensional state space and reasoning about the motions of multiple robots operating in the shared task space.}
    \label{fig:teaser}
\end{figure}

Our contributions in this paper are threefold. First, we introduce a novel method for M-RAMP that leverages the Conflict-Based Search framework and effectively re-uses intermediate search efforts to accelerate the search. Second, we provide a comprehensive theoretical analysis of our proposed framework, demonstrating its bounded sub-optimality guarantees.
Third, we offer an empirical evaluation of our method and other algorithms in various multi-arm manipulation scenarios that include deadlock avoidance, cluttered environments, and closely interacting goals.

\section{Related Work}
\label{sec:related_work}
The literature has extensively examined path planning for both single and multiple agents. In the context of single-agent search, decades of research have yielded algorithms capable of scaling successfully to high-dimensional and computationally expensive search spaces. However, efforts in MAPF have generally been applied to domains such as 2D-grid worlds, resulting in algorithms that often rely on assumptions such as fast single-agent planning and informative heuristics. These assumptions may not always hold in other scenarios such as robotic manipulation. In this section, we discuss relevant work in MAPF, describe common approaches to M-RAMP in practice, and review the use of experiences in planning that inspired us to connect MAPF algorithms with multi-arm manipulation more effectively.

\subsection{Multi-Agent Path Finding}

MAPF is the problem of finding collision-free paths for a set of agents on a graph (e.g., on a grid world) \cite{stern2019multi}. MAPF has been studied extensively, and optimal (e.g., CBS \cite{sharon2015conflict}), bounded sub-optimal (e.g., ECBS \cite{barer2014suboptimal}), and sub-optimal but effective (e.g., MAPF-LNS2 \cite{li2022mapf}) algorithms have been proposed. Some work has also been done to generalize MAPF algorithms to non-point robots \cite{li2019multi}, however, the most common domain is still in 2D. Arguably, the most influential family of algorithms is Conflict-Based Search (CBS) and its extensions \cite{sharon2015conflict, barer2014suboptimal, li2021eecbs}. CBS is a two-level search algorithm, where at the low level, each agent is assigned a single-agent path-planning problem. At the high level, collisions between single-agent solutions are resolved by imposing constraints on the low-level planners.

CBS is known to provide completeness and optimality guarantees. However, CBS is also known to be computationally expensive as it requires repeated low-level searches upon additions of constraints. Given this inefficiency, CBS is often regarded as impractical for domains, such as manipulation, where planning for a single agent requires exploring a high-dimensional space and does not enjoy informed heuristics.
In this work, we capitalize on this repetition and propose a method for accelerating CBS-based algorithms by reusing online-generated previous search solutions. 

\subsection{Planning for Multi-Arm Manipulation}
In practice, planning for multi-arm manipulation is often done with coupled or prioritization methods. In coupled methods, the state of all arms is seen as a single composite state, and the search is performed in this space with algorithms such as A* \cite{a*}, Rapidly-exploring Random Trees (RRT) \cite{RRT*}, and their variants (e.g., weighted A*, RRT* \cite{RRT*}, and RRT-Connect \cite{RRT-Connect}). With the addition of more arms, the search space grows exponentially, and in general, coupled methods may not scale to large numbers of arms.

In scenarios where coupled planning is rendered intractable due to the exponential growth of the search space, prioritization methods may be effective in reducing its dimensionality. In Prioritized Planning (PP) \cite{erdmann1987multiple}, each arm is assigned a priority, and the lower-priority arms treat higher-priority arms as moving obstacles. In the general case, solving the prioritized planning problem is more efficient than solving the coupled case, as the search space is reduced to the space of each single arm. However, the price paid for this dimensionality reduction is the loss of completeness. In scenarios requiring close coordination between arms, completeness may be important.

Recently, planning algorithms have been proposed specifically for teams of high-dimensional agents and applied to multi-arm settings. These methods explore the search space via pre-constructed single-agent roadmaps (e.g., probabilistic roadmaps (PRM) \cite{PRM}),
which may need to be arbitrarily resampled \cite{solis2021representation} to find collision-free paths in complex environments.
\citet{shome2020drrt} present dRRT*, a method for exploring the composite state space of agents by traversing individual agents' roadmaps towards sampled configurations with goal bias. \citet{solis2021representation} present CBS-MP, a variant of CBS that imposes new constraints on the search to resolve collisions. Specifically, to resolve a collision between two agents, CBS-MP requires one agent to avoid the colliding configuration of the other at the time of the collision.

\subsection{Leveraging Experience in Planning}
Streamlining motion planning from experience encompasses a wide range of motion-planning algorithms. These generally benefit from either utilizing offline-generated data (i.e., precomputation), leveraging online-generated data, or both.  

\subsubsection*{Precomputation as Experience} 

The utilization of offline computations to enhance online search efficiency is well exemplified by the PRM algorithm and its variants. Another novel approach is found in the Constant-Time Motion Planners (CTMP) family of algorithms, which operates on precomputed data structures to achieve constant-time path generation in online scenarios \cite{ctmp, ConveyerCTMP, mishani2023constanttime}. In recent research, a significant focus has been on the offline decomposition of the configuration space into collision-free convex sets \cite{polyhedral-convex}. This decomposition enables planning smooth trajectories within these sets using optimization methods \cite{marcucci2022motion}. 
Furthermore, various algorithms based on precomputed trajectories \cite{egraphs}, 
have been employed to expedite the search process. When extending these techniques to plan for multi-arm setups, it becomes essential to decompose the composite configuration space for computing collision-free trajectories. 
However, challenges arise when the environment or the robot undergoes changes, which can be as simple as rotating a bin or altering the robot's geometry by grasping an item. These changes may require resource-intensive operations like redoing precomputation or propagating changes, emphasizing a drawback inherent to using offline-generated experiences. 

\subsubsection*{Online-Generated Experiences}
Online-generated experiences have been used in MAPF, but more commonly in anytime incremental heuristic search and reinforcement learning (RL). In the heuristic search, experiences are used to directly guide the search, and in RL, as replay buffers to improve learning stability and sample efficiency \cite{replay_buffer}.
In MAPF, one algorithm leveraging repetition in planning is Iterative-Deepening CBS (IDCBS) \cite{boyarski2021iterative} that employs Lifelong Planning A* (LPA*) \cite{koenig2004lifelong} for single-agent planning. However, this approach faces challenges in manipulation cases where bounded sub-optimal search is employed to navigate the high-dimensional search space \cite{likhachev2005generalized}. 

Anytime algorithms, like Anytime Repairing A* (ARA*) \cite{anytime_dynamicA*}, can be seen as methods that use experiences generated online to improve the quality of the solution over time. ARA* performs a sequence of searches that, given enough time, converge to the optimal solution. 
A recent anytime approach inspired by \citet{anytime_dynamicA*, egraphs} and presented in \citet{mishani2023constanttime} employs both precomputation and online experience. Their algorithm computes an initial, potentially sub-optimal path within a (short) constant time and improves the quality of the path using the current best solution as an experience.

Drawing inspiration from the way \citet{mishani2023constanttime} capitalize on the flexibility seen in online-generated experiences, and with the observation that CBS-based algorithms inherently exhibit repetition in the form of nearly identical single-agent planning queries, we propose a method for accelerating MAPF algorithms by reusing online-generated experiences.

\begin{figure*}[t]
    \centering
    \includegraphics[width=\linewidth]{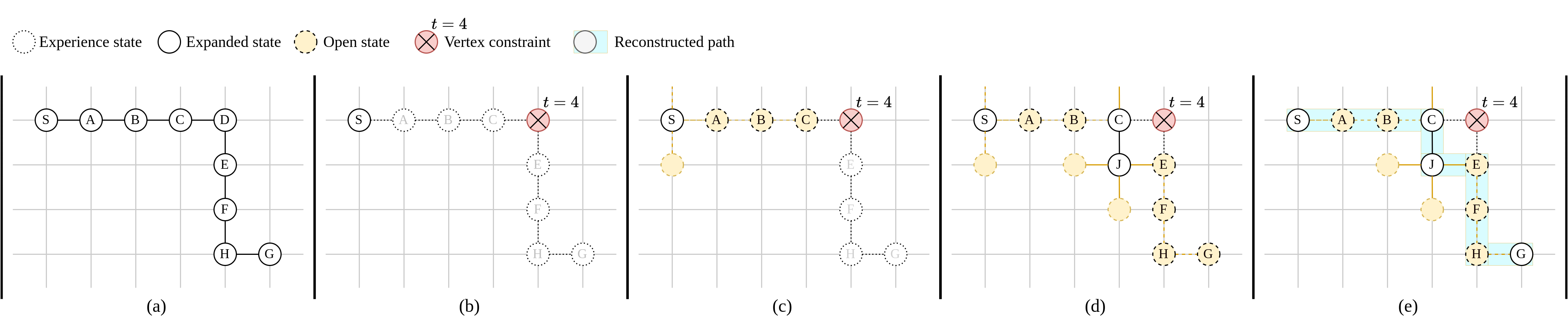}
    \caption{An illustration of our proposed algorithm accelerating a single agent search on a four-connected grid via reusing previous search efforts. 
    (a) A single-agent path from $S$ to $G$ computed in a previous iteration. 
    (b) Upon imposing a new constraint on the agent, shown in red, replanning is required. The previous path is drawn in light gray. 
    (c) Upon expansion of node $S$, a prefix $\{A, B, C\}$ of the experience path is added to OPEN along with all other successors of $S$. (d) shows two steps: node $C$ is selected for expansion from OPEN, and in the next iteration node $J$ is expanded from OPEN. Upon expanding $J$, a segment of the experience is added to OPEN, since one of $J$'s successors is equivalent to a node in the experience. 
    (e) Finally, $G$ is expanded from OPEN and the search terminates and recovers a path. In this example, the work done by xWA* (Alg. \ref{alg:low-level}) is smaller than that of its previous iteration. By reusing experience, the intermediate nodes expanded are $C$ and $J$.}
    \label{fig:low-level}
\end{figure*}

\section{Preliminary}

In this paper, we propose a method for solving the M-RAMP problem by extending the CBS algorithm and its variants to reuse search efforts. We first describe the problem formulation, explain how to cast it as a graph-search problem, and then detail how the CBS algorithm can solve M-RAMP.
\subsection{M-RAMP: Problem Formulation}

Consider $\confspace{i}{} \subseteq \mathbb{R}^d$ as the configuration space of a single robotic arm $\mathcal{R}_i$ with $d$ degrees of freedom (DoF).
A configuration $\conf{i}{} \in \confspace{i}{}$ is defined by assigning values to all the DoF, i.e., defining the joint angles.

Let the composite configuration space of $n$ robotic arms be $\confspace{}{} = \confspace{1}{} \times \confspace{2}{} \times \cdots\confspace{n}{}$. With all manipulators operating within the same environment $\mathcal{W} \subset \mathbb{R} ^ 3$, let $\confspace{}{\text{free}}$ be the set of all collision-free composite configurations (both with the environment and between robots): 
\[\confspace{}{\text{free}} = \{\compconf{} \in \mathcal{Q} \mid \compconf{} \text{ is collision-free}\}\]
Given an initial composite configuration $\compconf{\text{start}} \in \confspace{}{\text{free}}$ and a composite goal configuration $\compconf{\text{goal}} \in \confspace{}{\text{free}}$, 
we want to find a shortest valid path $\bar \Pi: [0,T] \rightarrow \confspace{}{\text{free}}$, in terms of per-robot time in motion, where $\bar \Pi(0) = \compconf{\text{start}}$ and $\bar \Pi(T) = \compconf{\text{goal}}$.

\subsection{M-RAMP as Graph Search}
A discrete analog of the M-RAMP problem is to find a sequence of composite configurations $\{\compconf{0}, \compconf{1}, \cdots, \compconf{T}\}$ such that $\forall t \in \{0, \cdots T\}, \; \compconf{t} \in \confspace{}{\text{free}}$, each interpolated configuration between $\compconf{t}$ and $\compconf{t+1}$ is collision-free, and $\compconf{0} = \compconf{\text{start}}$ and $\compconf{T} = \compconf{\text{goal}}$. Instead of addressing the motion planning problem in the high-dimensional composite state space, it is possible to decompose the problem into a set of single-agent motion planning problems each computing a collision-free path $\pi^i$ for $\robot{i}$ between $\confstart{i}$ and $\confgoal{i}$. 

One way to realize single-agent planning is by creating a lattice of allowable motions (often a small rotation in a single joint) and encoding those in an implicit graph for each robot $\robot{i}$ to represent its configuration space. Vertices correspond to configurations $\conf{i}{}$ and edges to unit-time transitions between configurations \cite{cohen2011adaptiveprimitives}. Between time steps, agents can traverse an edge or wait at a vertex. 

When dealing with robot arms that have complex search spaces 
it is generally infeasible to enumerate all collision-free vertices and edges before the search or to compute informative heuristics.
Each validity check requires querying a \textit{collision-checker}, which determines the arms' occupancy in $\mathcal{W}$ via forwards kinematics and checks for collisions. This operation is usually computationally expensive. 
Therefore, it is common to determine edge and vertex validity during the search itself and,S to further reduce the number of collision checks, employ weighted A* for the single-agent search. 

This approach allows us to find a path $\pi^i = \{\conf{i}{0}, \conf{i}{1}, \cdots, \conf{i}{T_i} \}$ for each robot that avoids collisions with obstacles, but potentially not with each other. Arms may collide during edge traversals (named \textit{edge-conflict}), and similarly when at vertex configurations (named \textit{vertex-conflicts}).
To check for validity, the individual paths can be combined into a multi-agent solution $\Pi = \{\pi^1, \pi^2, \cdots, \pi^n\}$ or as a composite path by merging the individual robot configurations at each time step, i.e., $\compconf{t} = \{\conf{1}{t},\cdots,\conf{n}{t}\}$. 
One way to find a valid (collision-free) solution of least sum of costs $|\Pi| = \sum_i |\pi^i| = \sum_i \sum_{t'=0}^{T_i-1} \text{cost}(\conf{i}{t'}, \conf{i}{t'+1})$ is to use CBS.

\subsection{Conflict-Based Search}
\label{prelims:cbs}
CBS is a two-level best-first search algorithm designed to solve the MAPF problem. Although commonly applied to 2D grids, it can handle any discrete time problem where agent motion is on a graph. 
CBS identifies and resolves conflicts between agents by imposing motion constraints and replanning paths accordingly.

CBS begins by querying a path $\pi_i$ for each agent $\robot{i}$ between its start and goal configurations without regard to other agents. This solution $\Pi = \{\pi^0,\cdots,\pi^n\}$ is a (possibly invalid) candidate solution for the problem, and it is stored in the OPEN list of the high-level search. 
Each high-level node $N$, known as a Constraint-Tree (CT) node, contains constraints $N.C$ on low-level planners and paths $N.\Pi$ that adhere to those constraints for all agents. 
The cost of a CT node $N$.cost is the sum of the costs of its stored paths $|N.\Pi|$.

CBS proceeds iteratively, selecting least-cost CT nodes $N$ from OPEN and evaluating them for conflicts. If no conflicts are found in a solution, it is accepted as valid, and the algorithm terminates. Otherwise, CBS chooses a conflict and resolves it by splitting $N$ into two child nodes that are added to OPEN. In each child node, one agent participating in the conflict is prohibited from using the edge or vertex it used during the conflict. This is imposed via \textit{edge-} or \textit{vertex-constraints}.
In M-RAMP, where each agent operates on their own graph, edge-conflicts take the form
$\langle i, \conf{i}{t}, \conf{i}{t+1}, t \rangle$ or $\langle j, \conf{j}{t}, \conf{j}{t+1}, t \rangle$, forbidding $\robot{i}$ or $\robot{j}$ from moving on the edge between the configurations $\conf{i}{t}$ and $\conf{i}{t+1}$ or $\conf{j}{t}$ and $\conf{j}{t+1}$ at time $t$. \textit{Vertex-constraints} $\langle i, \conf{i}{t}, t \rangle$, $\langle j, \conf{j}{t}, t \rangle$ forbid agents from visiting a configuration at time $t$.\footnote{Assuming complete low-level searches, CBS-based algorithms are complete as long as, when they create constraints $c_1$ and $c_2$ for resolving a conflict, there won't exist any valid solution that invalidates both $c_1$ and $c_2$. Otherwise, valid solutions with respect to conflicts will be marked as invalid against constraints. Interestingly, by viewing the completeness of CBS in this way, it can be shown that 
some CBS variants, such as CBS-MP, gain efficiency by imposing stronger constraints but sacrifice completeness. We have discussed CBS-MP's theoretical guarantees with the authors and reached this conclusion.}

\subsection{Bounded CBS (BCBS)}
BCBS \cite{barer2014suboptimal} is a bounded sub-optimal variant of CBS that trades off optimality in favor of efficiency. It employs \textit{focal search}, an algorithm based on A$^*_\epsilon$ \cite{pearl1982astarepsilon}, in its low- and high-level searches. Focal search employs two priority queues: OPEN and FOCAL. OPEN mirrors the A* queue and is sorted by an admissible priority function $f_1$. FOCAL comprises a subset of OPEN defined as FOCAL $=\{N \in \text{OPEN} \mid f_1(N) \leq w \cdot \min_{N' \in \text{OPEN}}f_1(N')\}$ with $w \ge 1$ a sub-optimality bound. The next node expanded is chosen from FOCAL according to priority function $f_2$.
Denoting search levels with a superscript, the BCBS high-level sets $f_1^H$ to the sum of costs of CT nodes $N$ and $f_2^H$ to the number of conflicts in $N$. At the low-level, $f_1^L$ is the original priority function of the search, and $f_2^L$ prioritizes nodes whose current path on the search tree has fewer conflicts against the preexisting paths of other agents. BCBS guarantees $w^Lw^H$-sub-optimal results.

\subsection{Enhanced CBS (ECBS)}
ECBS \cite{barer2014suboptimal} is an improvement on BCBS, producing solutions bounded by a single factor $w$.
To do so, the low-level in ECBS, when planning for $\robot{i}$, returns a path $\pi^i$ and a lower bound on the optimal cost $C^{*,i}$ of this path $lb(i) = \min_{s \in \text{OPEN}} f(s) \leq C^{*,i}$. 
ECBS keeps track of a lower bound $LB(N)$ on the sum of costs for each CT node $N$ defined as $LB(N) = \sum_{i=1}^{n} lb(i)$. Defining $LB = \min_{N' \in \text{OPEN}} LB(N') \leq C^*$ being a lower bound on the optimal sum of costs, and constructing the high-level FOCAL with FOCAL $=\{N \in \text{OPEN} \mid |N.\Pi| \leq w \cdot LB\}$, CT nodes in the high-level FOCAL satisfy $|N.\Pi| \leq w \cdot LB \leq w \cdot C^*$. 

In CBS and its variants, consecutive low-level searches are nearly identical. For example, a low-level planner for agent $\robot{i}$ invoked with constraints $C_i = \left\{c \in C \; | \; c \text{ involves } \robot{i} \right\}$ may next be invoked with $C_i \cup \left\{ \langle i, \conf{i}{t}, t \rangle \right\}$ after a single new constraint is added. This minor difference suggests potential benefits from reusing parts of the previous solution. 

\label{Prob}

\section{Algorithmic Approach}
\label{alg_app}

\begin{algorithm}[!t]
    \caption{High-Level (HL) Planner}  
    \label{alg:high-level}
    \footnotesize
    \SetKwInOut{Input}{Input}
    \SetKwInOut{Output}{Output}
    \SetKwFunction{LLPlanner}{\scriptsize LLPlanner}
    \SetKwFunction{Solve}{\scriptsize Solve}
    \SetKwFunction{GetCost}{\scriptsize GetCost}
    \SetKwFunction{InitRootNode}{\scriptsize InitRootNode}
    \SetKwFunction{RemoveTime}{\scriptsize RemoveTime}
    \SetKwFunction{FindConflicts}{\scriptsize FindConflicts}
    \SetKwFunction{GetSuccessors}{\scriptsize GetSuccessors}
    \SetKwFunction{GetConstraints}{\scriptsize GetConstraints}
    \SetKwFunction{ConflictsToConstraints}{\scriptsize ConflictsToConstraints}
    \SetKwFunction{Copy}{\scriptsize Copy}
    \SetKwFunction{GetLowerBound}{\scriptsize GetLowerBound}
    
    \Input{
        $\confstart{} = \{\confstart{1}, \dots, \confstart{n} \}$ \newline
        $\confgoal{} = \{\confgoal{1}, \dots, \confgoal{n} \}$ \newline
        $f_1^H, f_2^H$: priority functions, \newline
        $w^H$: sub-optimality bound.
        }
    \Output{$\Pi = \{\pi^1, \cdots, \pi^n\}$}
    \SetAlgoLined\DontPrintSemicolon

    \SetKwFunction{procroot}{InitRootNode}
    \SetKwProg{myprocroot}{Procedure}{}{}

    \SetKwFunction{proc}{Plan}
    \SetKwProg{myproc}{Procedure}{}{}

    \vspace{5pt}
    \myprocroot{\procroot{}}{
    $N_\text{root}$.$C$ $\leftarrow$ $\emptyset$ \;
    $N_\text{root}$.$\Pi$ $\leftarrow$ invoke \LLPlanner for each agent \; 
    $N_\text{root}$.cost $\leftarrow$ $|N_\text{root}$.$\Pi|$ \;
    \KwRet $N_\text{root}$\;  
    } 
    \vspace{5pt}

    \myproc{\proc{$\confstart{}$, $\confgoal{}$}}{ 
    $N_\text{root}$ $\leftarrow$ \InitRootNode{} \;
    OPEN $\gets$ $\{N_\text{root}\}$\;
    \While{OPEN not empty}{
    $B \gets \min\limits_{N' \in \text{OPEN} } f_1^H(N') $ \;
    FOCAL $\gets \{N \in$ OPEN $\mid N.\text{cost} \leq w^H \cdot B \}$ \label{line:hl_open_lb} \; 
    $N \gets \;^{\; \arg \min}_{N' \in \text{FOCAL}} f_2^H(N')$\; \vspace{0.1cm} 

    OPEN.remove($N$)\;
    \If{$N$.conflicts $=\emptyset$}{ \label{line:check_no_conflicts}
        \KwRet $N$.$\Pi$ \label{line:hl_return_sol} \;
    }
    constraints $\leftarrow$ \GetConstraints ($N$.conflicts.first) \label{line:conflicts_to_constraints}\;
    \For{c $\in$ constraints} {  \label{line:forloop_constraints}
        New CT node $N' \gets$ \Copy($N$) \label{line:create_ct_node} \;
        \textit{$N'$}.$C$ $\leftarrow$ $N$.$C$ $\cup$ c \label{line:add_constraint_to_new_ct_node} \;
        $i \gets c.\text{agent\_id}$ \;
        \textit{Experience} $\leftarrow$ $N$.$\Pi[i]$ \label{line:exp_rem_time} \;
        $N'.\Pi[i]$ $\gets$ \LLPlanner.\Solve(\newline $\confstart{i}$, $\confgoal{i}$, $N'$.$C$, \textit{Experience}) \label{line:low_level_replan}\;
        \tcp*{Invoke \LLPlanner for each agent involved}
        \textit{$N'$}.cost $\leftarrow$ $|N'.\Pi|$\;
        $N'$.conflicts $\leftarrow$ \FindConflicts($N'$.$\Pi$) \; 
        OPEN.insert($N'$)\;
        }
    }
    \KwRet $\emptyset$ \;
    }
\end{algorithm}

\begin{algorithm}[!t]
    \caption{xWA$^*$: Low-Level (LL) Planner}
    \label{alg:low-level}
    \footnotesize
    \SetKwInOut{Input}{Input}
    \SetKwInOut{Output}{Output}
    \SetKwFunction{IsEdgeValid}{\scriptsize IsEdgeValid}
    \SetKwFunction{IsStateValid}{\scriptsize IsStateValid}
    \SetKwFunction{suffix}{\scriptsize suffix}
    \SetKwFunction{RemoveTime}{\scriptsize RemoveTime}
    \SetKwFunction{PropagagteTimeAndCost}{\scriptsize PropagagteTimeAndCost}
    \SetKwFunction{TryInsertOrUpdate}{\scriptsize TryInsertOrUpdate}
    \SetKwFunction{PushPartialExperience}{\scriptsize PushPartialExperience}
    \SetKwFunction{ExtractPath}{\scriptsize ExtractPath}
    \SetKwFunction{NoFutureConstraints}{\scriptsize NoFutureConstraints}
    \SetKwFunction{IsGoalCondition}{\scriptsize IsGoalCondition}
    \SetKwFunction{GetSuccessors}{\scriptsize GetSuccessors}

    \Input{$q_{start}$: start ($q_{start} \in \mathcal{Q}^{free}$), \newline
    $q_{goal}$: goal ($q_{goal} \in \mathcal{Q}^{free}$), \newline
    $C$: constraints set, \newline
    $\tilde{\pi}$: experience sequence (without time), \newline
    $w_1^L, w_2^L$: sub-optimality bound in WA*, focal list, \newline
    $f_1^L := g(s) + w_1^L h(s), f_2^L$: priority functions.}
    \Output{Path $\pi$}
    \SetAlgoLined\DontPrintSemicolon

    \SetKwFunction{procupd}{TryInsertOrUpdate}
    \SetKwProg{myprocupd}{Procedure}{}{}
    \vspace{5pt}
    \myprocupd{\procupd{$s_1$, $s_2$, OPEN}}{
        \If{$s_2$ was not visited before}{$g(s_2) \gets \infty$\;}
            \If{$g(s_2) > g(s_1) + \text{cost}(s_1, s_2)$}{
                $g(s_2) \gets g(s_1) + \text{cost}(s_1, s_2)$ \;
                OPEN.InsertOrUpdate($s_2$)\;
        }
    }

    \SetKwFunction{procexp}{PushPartialExperience}
    \SetKwProg{myprocexp}{Procedure}{}{}
    \vspace{5pt}
    \myprocexp{\procexp{$\tilde{\pi}$,$s$,$C$,OPEN}}{
    $(q_0,t_0) \gets s$\;
    $\hat{\pi} \gets \tilde{\pi}.\text{\suffix}(q_0)$; 
    \tcp*{The experience configurations after $q_0$.}
    \For{$\hat q \in \hat{\pi}$}{
        $\hat{s} \gets (\hat q, t_0+1)$ \label{line:xwa_propagate_time_values_from_s}\;
        \eIf{\IsEdgeValid($(q_0,t_0)$, $\hat{s}$,$C$) $\wedge$ \IsStateValid($\hat{s}$,$C$) \label{line:xwa_add_from_exp_if}}{
            \TryInsertOrUpdate($(q_0,t_0)$, $\hat{s}$, OPEN) \label{line:xwa_insert_or_update_open}\;
            $(q_0, t_0) \leftarrow \hat{s}$\;
            }{break\;}
        }
    } 
    \vspace{5pt}
    
    \SetKwFunction{procrefine}{Solve}
    \SetKwProg{myproc}{Procedure}{}{}
    \myproc{\procrefine{$q_{start}$, $q_{goal}, C$, $\tilde{\pi}$}}{
        $s_\text{root} \leftarrow (q_{start}, 0)$ \tcp*{Adding time to state}
        OPEN $\gets$ \{$s_\text{root}$\} \; 
        \PushPartialExperience{$\tilde{\pi}$, $s_\text{root}$, $C$, OPEN} \label{line:xwa_push_partial_exp_start}\;
        
        \While{OPEN $\neq \emptyset$}{
            FOCAL $\gets \{s \in OPEN \mid f_1^L (s) \leq w_2^L \min\limits_{s' \in \text{OPEN} } f_1^L(s') \}$ \; 
            \vspace{-0.25cm} 
            $s = (\conf{}{}, t ) \gets \;^{\; \arg \min}_{s' \in \text{FOCAL}} f_2^L(s')$ \label{line:s_min_focal}\; \vspace{0.05cm} 
            OPEN.remove($s$)\;
            \If{$\conf{}{} = \confgoal{}$ $\wedge$ no future constraints at $\confgoal{}$
            \label{line:wxa_s_check_goal_equiv}
            }
            { 
                \KwRet $\pi \leftarrow \ExtractPath(s)$ \label{line:xwa_return_pi} \;
            }
            \If{$\conf{}{} \in \tilde{\pi}$}{ \label{line:xwa_s_in_exp_check}
                \PushPartialExperience{$\tilde{\pi}$, $s$, C,  OPEN} \label{line:xwa_push_partial_exp}\;
            }
            
            \For{$s' \in \GetSuccessors(s, C)$\label{line:xwa_get_successors}} {
                \TryInsertOrUpdate{$s$, $s'$, OPEN} \label{line:xwa_insert_s_to_open} \;
            }
        }
    \KwRet $\emptyset$ \label{line:xwa_return_empty}\;
    }
\end{algorithm}
Our main contribution in this work is an experience-acceleration framework for CBS-based algorithms.
We instantiate this framework in two incarnations, \textit{xCBS} and \textit{xECBS}, accelerating CBS and ECBS, respectively. In this section, we present the general form of our acceleration method in an intuitive manner grounded by Algorithm \ref{alg:high-level} and Algorithm \ref{alg:low-level} and then provide a theoretical analysis of its performance alongside its instantiations xCBS and xECBS. 

\subsection{Experience-Acceleration Framework}
\label{sec:alg_app_xcbs}

Our framework follows the CBS structure and informs new low-level planner calls with the experience generated in previous search efforts.
In the high-level search (Alg. \ref{alg:high-level}), each CT node contains a set of paths $\Pi$, one for each agent, and a set of constraints $C$. Upon obtaining a new node from a priority queue, it is checked for conflicts (line \ref{line:check_no_conflicts}). If none exist, the node is a goal node, and the paths are returned (line \ref{line:hl_return_sol}). Otherwise, a pair of constraints is derived from the first conflict (line \ref{line:conflicts_to_constraints}). Usually, CBS proceeds by creating a new CT node, one with an added constraint from the constraint set (lines \ref{line:forloop_constraints}-\ref{line:add_constraint_to_new_ct_node}), and replans a single-agent path for the affected agent from scratch (line \ref{line:low_level_replan}). However, we recognize that a considerable portion of the previously generated path remains valid and can be effectively reused. Thus, our high-level search caches a copy of the previously computed path as \textit{experience} and passes it to the low-level motion planner (lines \ref{line:exp_rem_time}, \ref{line:low_level_replan}). The experience path is a sequence of configurations (including waits and cycles). 
It is possible to construct the experience set for a replanned agent $\robot{i}$ in multiple ways. We have experimented with reusing its previous path stored in its parent CT node, paths for $\robot{i}$ from all previous searches on the CT branch, and all paths for $\robot{i}$ across the CT. Reusing the previous path performed the best.

The low-level of our acceleration framework, namely \textit{xWA$^*$}, is detailed in Algorithm \ref{alg:low-level} and illustrated in Fig. \ref{fig:low-level}. Let a \textit{state} be a configuration with time. Each state expansion (lines \ref{line:s_min_focal}-\ref{line:xwa_insert_s_to_open}) adds a set of successors to the OPEN list. 
Upon a choice of a state for expansion (line \ref{line:s_min_focal}), the search terminates if it is a goal state (lines \ref{line:wxa_s_check_goal_equiv}-\ref{line:xwa_return_pi}). Otherwise, we check if the expanded state belongs to the experience path (line \ref{line:xwa_s_in_exp_check}). 
If the expanded state belongs to an experience, 
starting from that state, we aim to add as much of the experience as possible to the OPEN list (line \ref{line:xwa_push_partial_exp}). This process is also applied to the start state (line \ref{line:xwa_push_partial_exp_start}) and essentially provides a ``warm start'' to the search effort.
Given an expanded state $s$ that belongs to an experience, we attempt to add consecutive states from the experience while propagating their associated time (line \ref{line:xwa_propagate_time_values_from_s}) and setting or updating their associated cost (line \ref{line:xwa_insert_or_update_open}). We continue this process until a \textit{termination condition} is met (line \ref{line:xwa_add_from_exp_if}). 
A simple condition terminates addition when it violates constraints in $C$. A more complex condition, for example, terminates the addition of states when transitions lead to a collision with another agent's path. This condition is effective for xECBS.

The effect of adding an experience path to the OPEN list of a bounded sub-optimal search algorithm, such as weighted A*, could be a rapid exploration of states that are closer to the end of the experience path (and, consequently, closer to the goal). Fig. \ref{fig:low-level} illustrates this effect. Such exploration results in the algorithm ``jumping" over previously explored regions and avoiding redundant search efforts, directing its focus closer to the end of the experience.

Collision checking against the static environment, a significant factor in the slowness of planning for manipulation, can also be directly accelerated with experience. To this end, our acceleration framework also keeps track of the configurations ($\conf{i}{t}, \conf{i}{t+1})$ in all valid transitions $(s_t, s_{t+1})$ for each robot $\robot{i}$. With this information, the successors set (line \ref{line:xwa_get_successors}) can be computed more rapidly by only checking the validity of edges previously unseen. 
Because one single-agent search can revisit the same configuration at different times, such experience reuse also speeds up the first search.

\subsection{Theoretical Analysis}
In this section, we discuss the theoretical foundation of our algorithm.
We formally define the problem for both levels of CBS as focal search and introduce some of the properties of CBS and its bounded sub-optimal variants.
We show that accelerating CBS variants by reusing experience retains completeness and bounded sub-optimality guarantees.

We commence by establishing the bounded sub-optimality of the low-level planner xWA$^*$ that leverages past experiences. To allow for the use of inflated heuristics using a weighted OPEN list \cite{Rishi}, which is common in manipulation, we expand our analysis to low-level planners with $w_1$\textit{-admissible} \cite{pearl1982astarepsilon} priority function $f_1(s) = g(s) + w_1 h(s)$.

\begin{lem}
\label{lem:eps-admissible}
    A focal search employing a \textit{$w_1$-admissible} function $f_1(s)$ ($w_1 \geq 1$) and $\focal=\{s \in \text{OPEN} \mid f_1(s) \leq w_2 \min\limits_{s' \in OPEN} f_1(s')\}$ has a sub-optimality factor $w_1 \cdot w_2$.
\end{lem}

\begin{proof}
    Let $s_0$ be a node on an optimal path that resides in OPEN. For every expanded node $s'$:
    \[f_1(s') \leq w_2 \min\limits_{s \in \open}f_1(s) \leq w_2 f_1(s_0) =\]
    \[w_2(g(s_0) + w_1 h(s_0)) \leq w_2 w_1 (g(s_0) + h(s_0)) \leq w_2 w_1 C^*\qedhere\]
\end{proof}

\noindent Next, we show that incorporating experiences in xWA$^*$ neither impacts its sub-optimality nor sacrifices completeness.

\begin{lem}
\label{lem:adding-nodes}
    Consider a best-first search storing frontier states in an OPEN list.
    When systematically incorporating successors into OPEN, if additional nodes are introduced along with their associated priority function values, completeness and bounded sub-optimality persist.
\end{lem}
\begin{proof}
    When introducing new nodes to OPEN, the original OPEN of weighted A* becomes a subset of the modified OPEN. Thus, the algorithm maintains its systematic nature, ensuring completeness.
    Furthermore, we also know that FOCAL will only be populated by nodes from OPEN that are within the specified sub-optimality bound.
    Consequently, when a goal state is expanded, the solution remains bounded sub-optimal.
\end{proof}

\begin{thm}
\label{thm:low-level-complete-bdd-suboptimal}
    xWA$^*$ is complete and bounded sub-optimal by a factor of $w^L = w_1^L w_2^L$.
\end{thm}
\begin{proof}
    Since xWA$^*$ is a $w_2^L$-sub-optimal focal search, which employs a weighted OPEN ($w_1^L$-admissible $f_1=f_1^L$), the proof follows directly from Lemmas \ref{lem:eps-admissible} and \ref{lem:adding-nodes}.
\end{proof}

We continue with analyzing the sub-optimality of the high-level planner, which is defined as a focal search. Let $f_1 = f_1^H$ be a priority function such that for every CT node $N$, $f_1^H(N) \leq N.\text{cost}$. 
Additionally, let the FOCAL queue be defined as $FOCAL = \{N \in \text{OPEN} \mid N.\text{cost} \leq w^H \cdot  \min\limits_{N' \in \text{OPEN} } f_1^H(N')\}$:

\begin{lem}
\label{lem:subopt-cbs-bound}
    Let $w^H, w^L$ be the sub-optimality factor of the high- and low-level focal searches, respectively. For any $w^H, w^L \geq 1$, the cost of a solution is at most $w^H w^L C^*$.
\end{lem}
\begin{proof}
    Let $N$ be a node in FOCAL of the high-level search. For each of the $\robot{i}$ of $n$ agents, we denote the returned cost of a low-level plan as $cost(i)$ and its optimal cost as $C^{*,i}$. 
        \[N.\text{cost} \leq w^H \min\limits_{N' \in \open}f_1^H(N')\leq w^H \min\limits_{N' \in \open}\sum\limits_{i=1}^{n}cost(i) \]
        \[\leq w^H \sum\limits_{i=1}^{n} w^L C^{*,i} = w^H w^L C^*\qedhere\]
\end{proof}

\begin{thm}
    Our proposed acceleration framework is complete and bounded sub-optimal.
    \label{thm:xcbs}
\end{thm}
\begin{proof}

    Building on the work of \citet{barer2014suboptimal}, we can establish that the high-level search is complete if the low-level planner is complete and the constraints are valid. Furthermore, it is bounded sub-optimal with a factor of $w^H$. As shown in Theorem~\ref{thm:low-level-complete-bdd-suboptimal}, xWA$^*$ is both complete and bounded sub-optimal by $w^L$. 
    By transitivity, our overall algorithm inherits completeness from both the high and low levels. Additionally, leveraging Lemma~\ref{lem:subopt-cbs-bound}, we conclude that the sub-optimality upper bound of our approach is $w^Hw^LC^*.\qedhere$
\end{proof}

Finally, we detail xCBS and xECBS as instances of our experience-acceleration framework and show their completeness and bounded sub-optimality.

\subsubsection*{\textbf{xCBS}} At the low- and high-level, xCBS does not use focal lists (i.e., $f_2^L = f_1^L, f_2^H = f_1^H, \; w^H=w^L_2 = 1$). Its CT node prioritization is identical to that of CBS ($f_1^H(N) = N.\text{cost}$), and so is its constraint generation function. Hence, its low- and high-levels are complete. By Theorem \ref{thm:low-level-complete-bdd-suboptimal}, it has a sub-optimality factor of $w^L = w_1^L$.
Thus, xCBS maintains completeness and is bounded sub-optimal by factor $w_1^L$.

\subsubsection*{\textbf{xECBS}} 
The low- and high-level focal lists are ordered similarly to ECBS, prioritizing nodes with fewer conflicts.
At the high-level, xECBS uses $f_1^H(N) = LB(N)$, keeping the sub-optimality factor $w^H=1$. At the low level, xWA* contributes a sub-optimality factor of  $w^L = w_1^L \cdot w_2^L$. Thus, its total sub-optimality factor is $ w_1^L \cdot w_2^L$.
Completeness is guaranteed for the same reasons as xCBS.

\begin{figure*}[t]
  \begin{minipage}{0.64\textwidth}
    \centering
      \includegraphics[height=4.5cm]{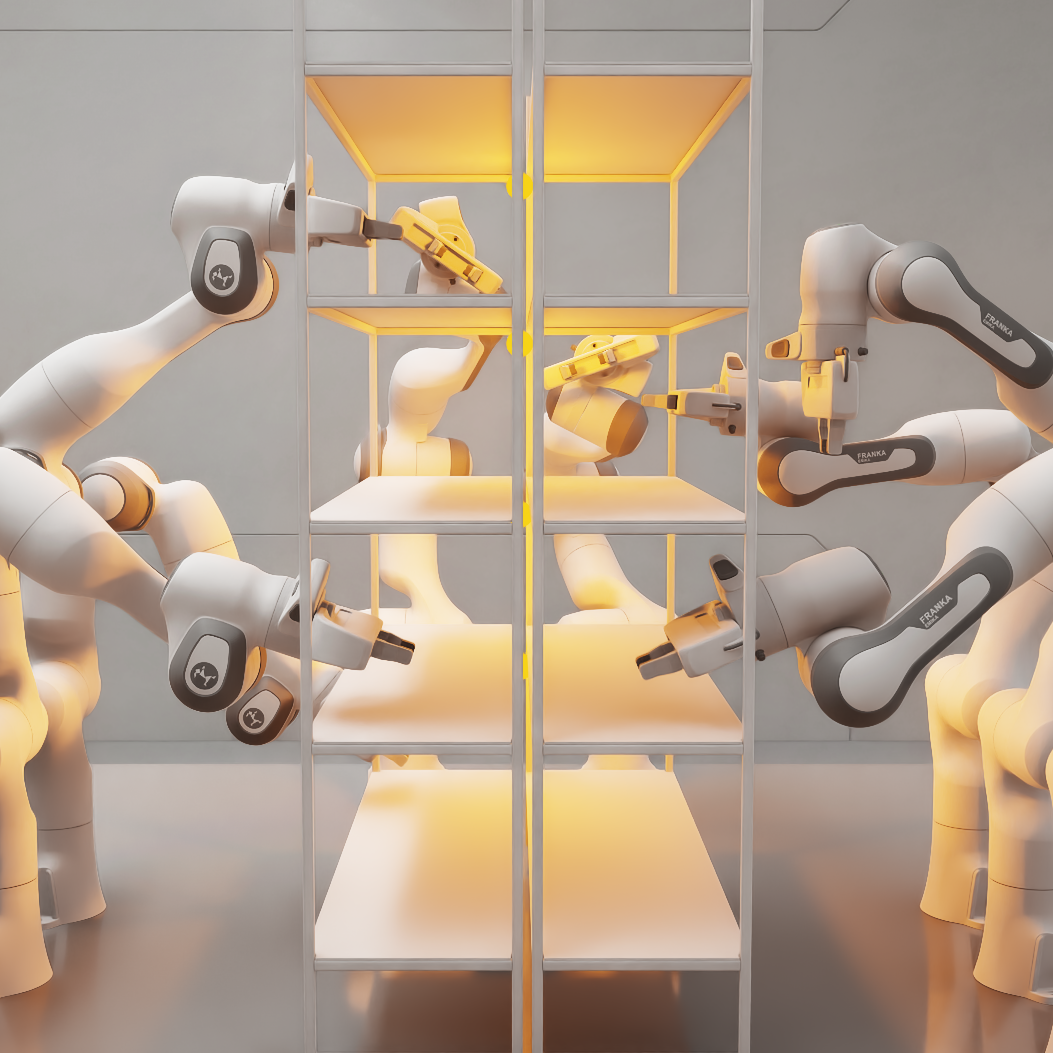}
      \hspace{0.3cm}
      \includegraphics[height=4.6cm, width=6.2cm]{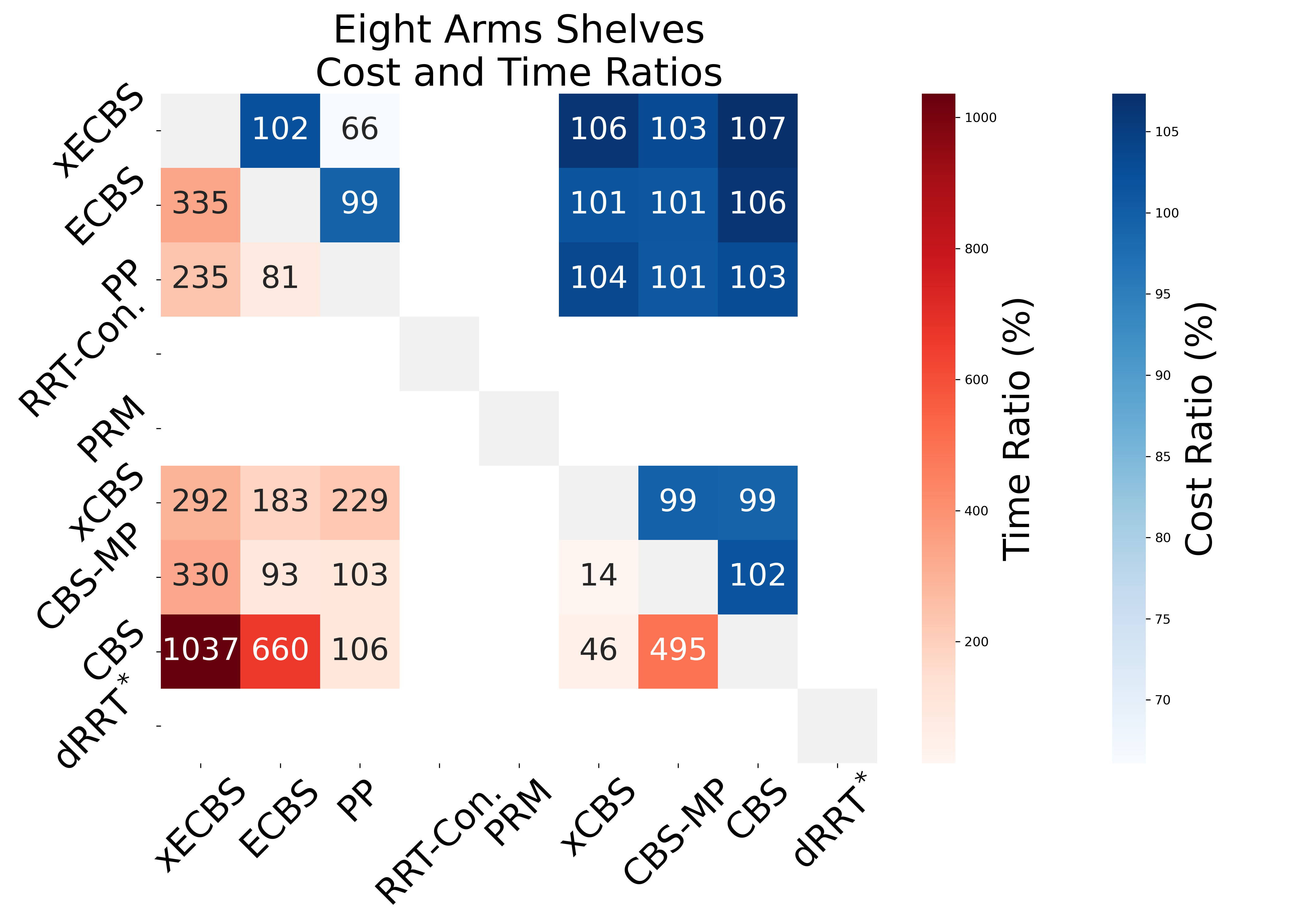}
  \end{minipage}
\hspace{-0.5cm}
\begin{minipage}{0.36\textwidth}
\resizebox{\textwidth}{!}{%
\begin{tabular}{|l|c|c|c|}
\hline
\multicolumn{1}{|c|}{}        & \cellcolor[HTML]{C0C0C0}Success & \cellcolor[HTML]{C0C0C0}Planning Time (sec)            & \cellcolor[HTML]{C0C0C0}Cost (rad) \\ \hline
\rowcolor[HTML]{EFEFEF} 
\cellcolor[HTML]{C0C0C0}xECBS & \textbf{84\%}                   & \cellcolor[HTML]{EFEFEF}\textbf{13.6 $\pm$ 12.1} & \textbf{41.9 $\pm$ 8.3}            \\ \hline
\cellcolor[HTML]{C0C0C0}ECBS     & 40\% & 26.9 $\pm$ 17.9 & 37.1 $\pm$ 6.3  \\ \hline
\cellcolor[HTML]{C0C0C0}PP       & 40\% & 30.0 $\pm$ 19.7 & 58.1 $\pm$ 50.6 \\ \hline
\cellcolor[HTML]{C0C0C0}RRT-Con. & x    & x               & x               \\ \hline
\cellcolor[HTML]{C0C0C0}PRM      & x    & x               & x               \\ \hline
\cellcolor[HTML]{C0C0C0}xCBS     & 4\%  & 34.8 $\pm$ 25.4 & 41.5 $\pm$ 8.9  \\ \hline
\cellcolor[HTML]{C0C0C0}CBS      & 4\%  & 39.8 $\pm$ 21.9 & 32.1 $\pm$ 4.9  \\ \hline
\cellcolor[HTML]{C0C0C0}CBS-MP   & 22\% & 25.7 $\pm$ 20.0 & 38.0 $\pm$ 4.6  \\ \hline
\cellcolor[HTML]{C0C0C0}dRRT*    & x    & x               & x               \\ \hline
\end{tabular}}
\end{minipage}

\hspace{0.2cm}

\begin{minipage}{0.64\textwidth}
\centering
  \includegraphics[width=4.5cm]{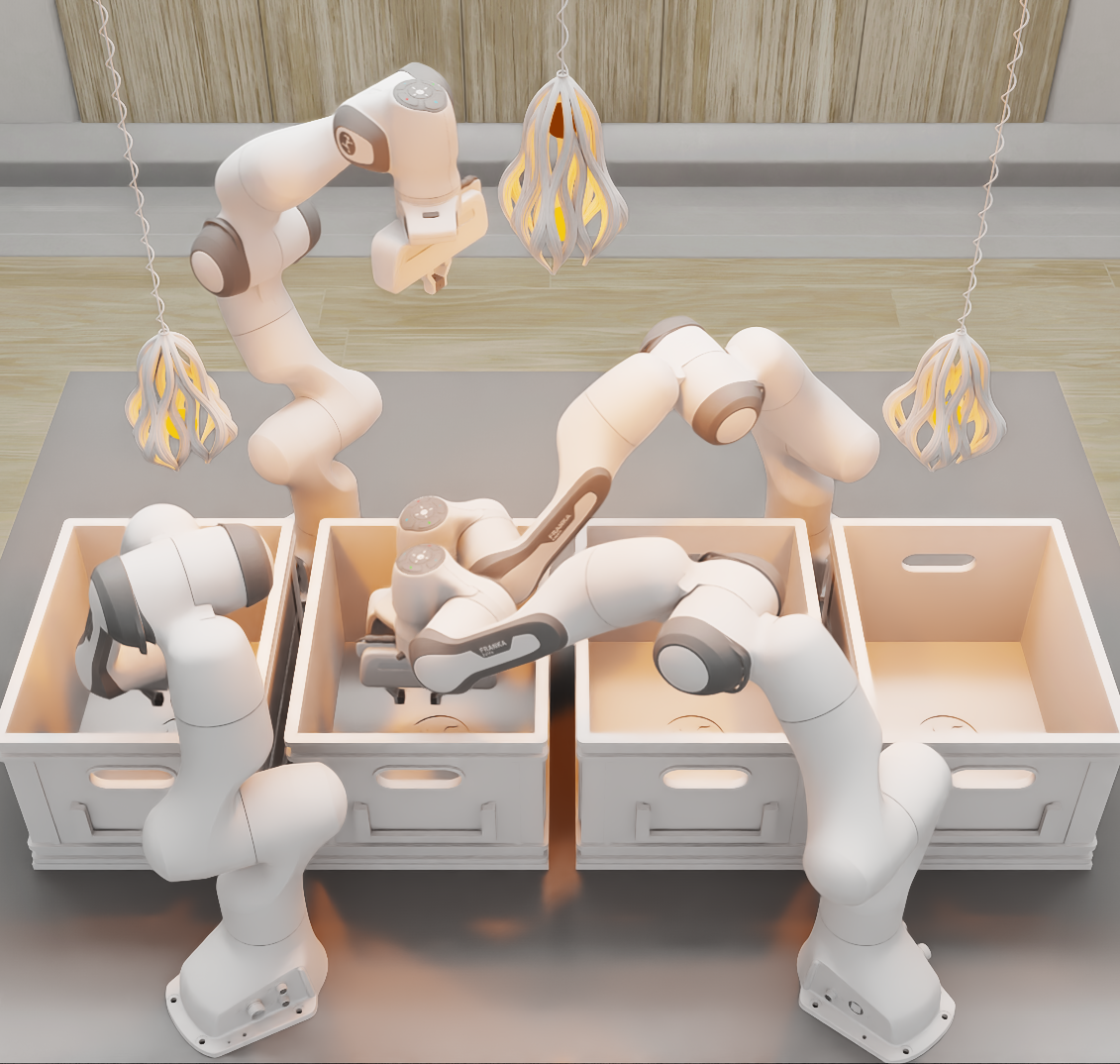}
  \hspace{0.3cm}
  \includegraphics[height=4.6cm, width=6.2cm]{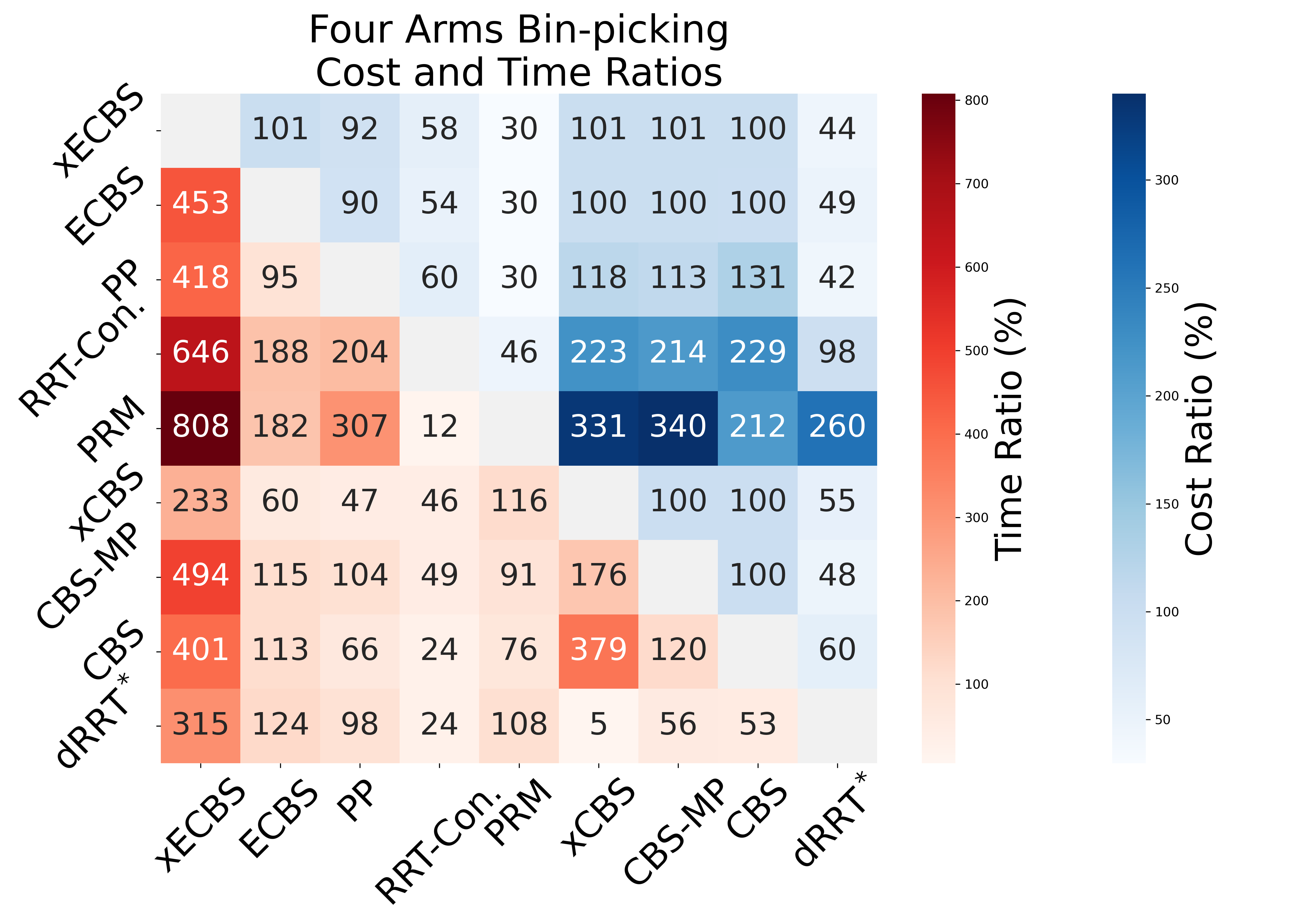}
\end{minipage}
\hspace{-0.5cm}
\begin{minipage}{0.36\textwidth}
\resizebox{\textwidth}{!}{%
\begin{tabular}{|l|c|c|c|}
\hline
                              & \cellcolor[HTML]{C0C0C0}Success & \cellcolor[HTML]{C0C0C0}Planning Time (sec) & \cellcolor[HTML]{C0C0C0}Cost (rad) \\ \hline
\rowcolor[HTML]{EFEFEF} 
\cellcolor[HTML]{C0C0C0}xECBS & \textbf{96\%}                   & \textbf{4.4 $\pm$ 3.4}                & \textbf{24.1 $\pm$ 3.5}            \\ \hline
\cellcolor[HTML]{C0C0C0}ECBS     & 82\% & 18.4 $\pm$ 16.9 & 23.3 $\pm$ 3.3  \\ \hline
\cellcolor[HTML]{C0C0C0}PP       & 84\% & 14.5 $\pm$ 17.1 & 25.4 $\pm$ 14.1 \\ \hline
\cellcolor[HTML]{C0C0C0}RRT-Con. & 42\% & 16.8 $\pm$ 14.2 & 42.0 $\pm$ 22.4 \\ \hline
\cellcolor[HTML]{C0C0C0}PRM      & 16\% & 9.3 $\pm$ 18.5  & 73.4 $\pm$ 46.2 \\ \hline
\cellcolor[HTML]{C0C0C0}xCBS     & 48\% & 8.6 $\pm$ 10.7  & 22.5 $\pm$ 2.9  \\ \hline
\cellcolor[HTML]{C0C0C0}CBS-MP   & 64\% & 14.4 $\pm$ 11.0 & 22.2 $\pm$ 2.6  \\ \hline
\cellcolor[HTML]{C0C0C0}CBS      & 28\% & 12.8 $\pm$ 11.2 & 21.8 $\pm$ 2.4  \\ \hline
\cellcolor[HTML]{C0C0C0}dRRT*    & 14\% & 6.1 $\pm$ 8.9   & 56.4 $\pm$ 20.8 \\ \hline
\end{tabular}%
}
\end{minipage}
  
  \caption{Evaluating the real-world applicability of planning algorithms. Left: evaluation scenes, with 8-arm shelf rearrangement and 4-arm bin-picking. Middle: Comparing planning time and cost for each row-name planner relative to the column-name planner. The values offer a fair comparison by considering only successful runs in both planners. For instance, xECBS has shorter planning times (red, above 100\%) and lower solution costs comparable to other CBS-based approaches (blue, around or below 100\%). xECBS is faster and finds short paths.
  Right: Success rate and mean$\pm$standard deviation among successful runs.}
  \label{fig:experimental_results}
\end{figure*}

\begin{figure*}[!h]

  \begin{minipage}{1.0\textwidth}
    \begin{minipage}{\textwidth}
      \centering
      \begin{minipage}{0.19\textwidth}
      \includegraphics[width=0.995\textwidth]{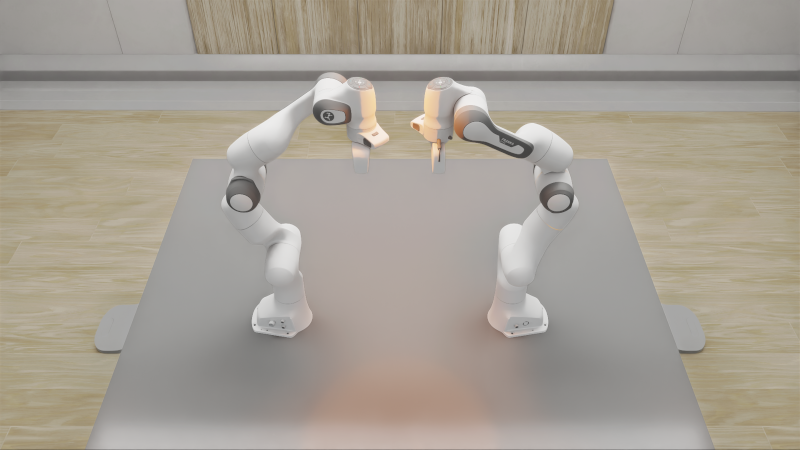}\\
      \end{minipage}
      \begin{minipage}{0.19\textwidth}
      \includegraphics[width=0.995\textwidth]{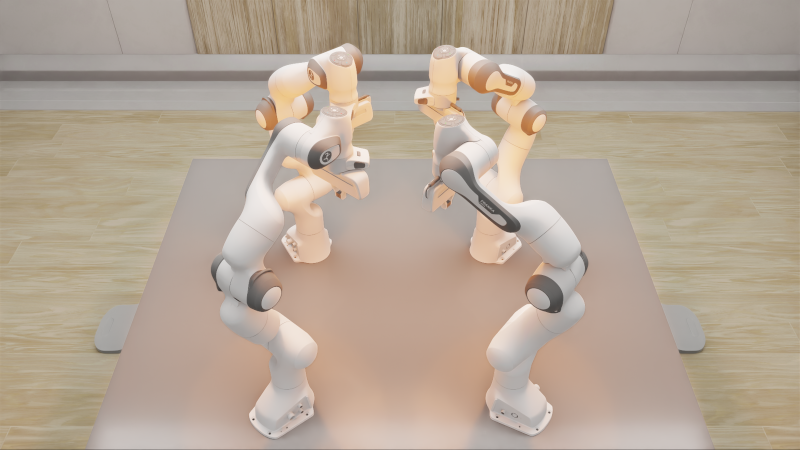}\\
      \end{minipage}
      \begin{minipage}{0.19\textwidth}
      \includegraphics[width=0.995\textwidth]{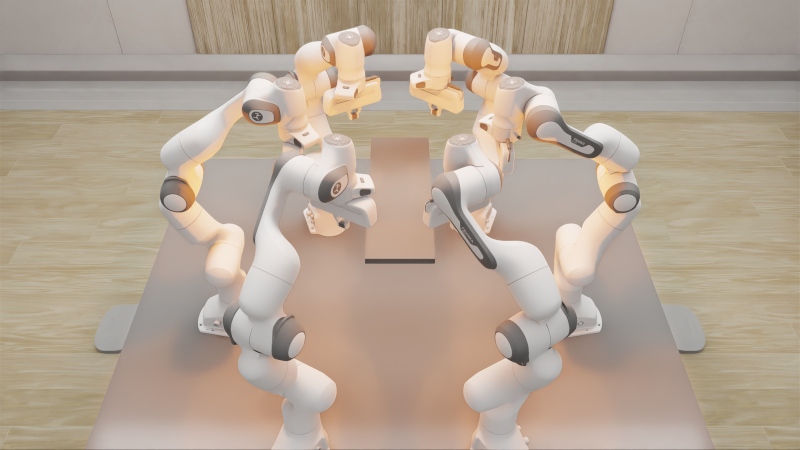}\\
      \end{minipage}
      \begin{minipage}{0.19\textwidth}
      \includegraphics[width=0.995\textwidth]{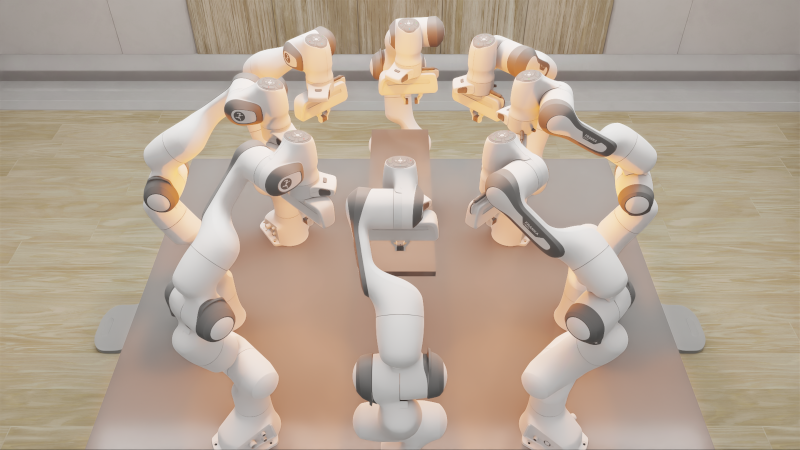}\\
      \end{minipage}
      \begin{minipage}{0.19\textwidth}
      \includegraphics[width=0.995\textwidth]{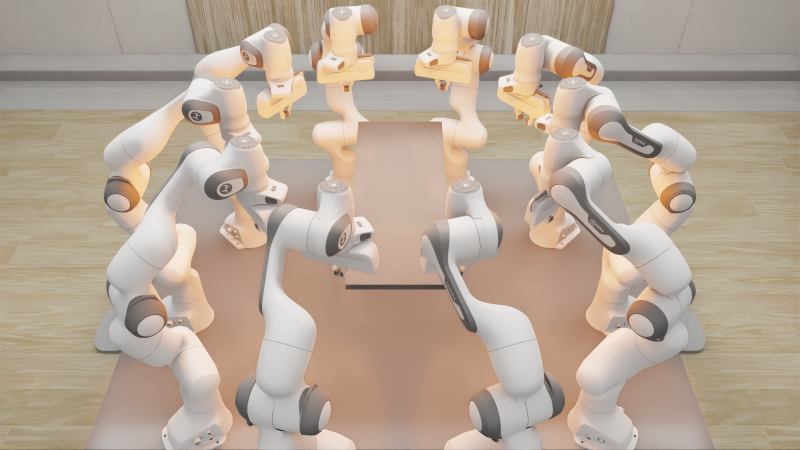}\\
      \end{minipage}
    \end{minipage}
  \end{minipage}
  
  \vspace{0.0cm}
  
  {
  \centering
  \begin{minipage}{1.0\textwidth}
    \centering
    \includegraphics[height=2.2cm]{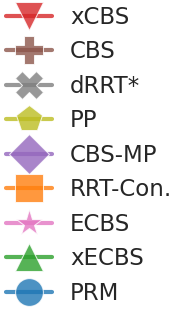}
    \includegraphics[width=0.2\textwidth]{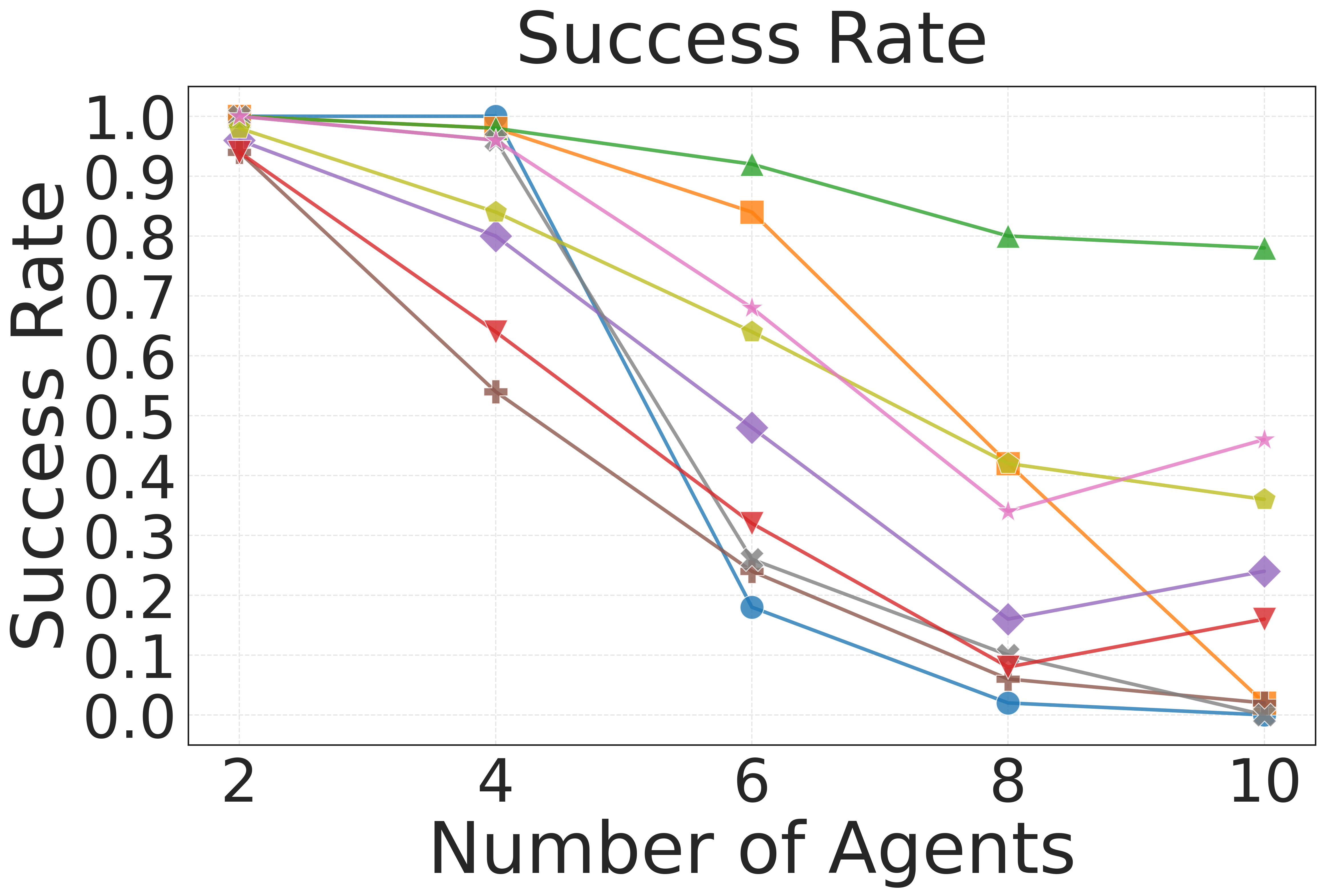} \hspace{0.3cm}
    \includegraphics[width=0.21\textwidth]{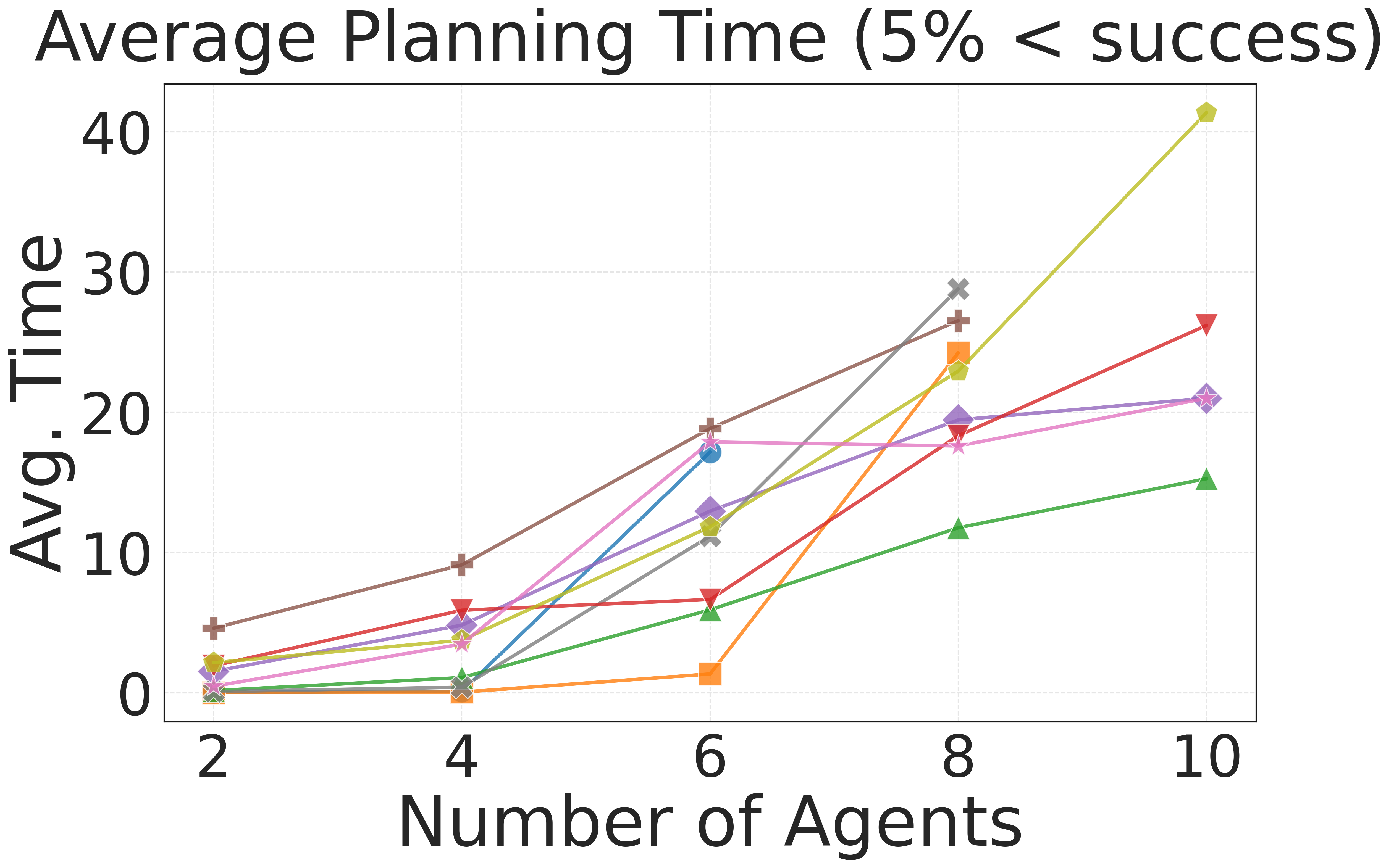} \hspace{0.3cm}
    \includegraphics[width=0.21\textwidth]{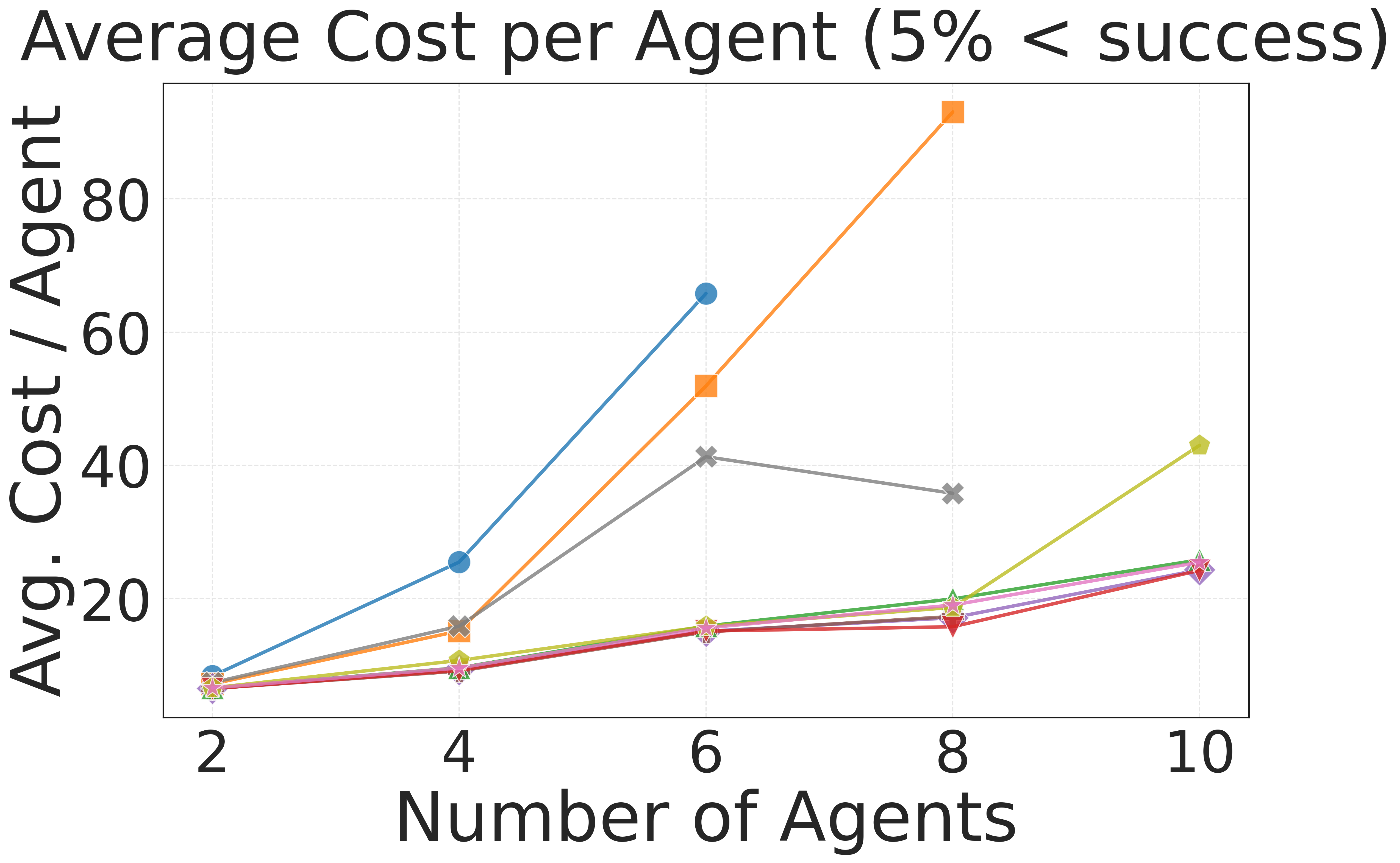} \hspace{0.3cm}
    \includegraphics[width=0.2\textwidth]{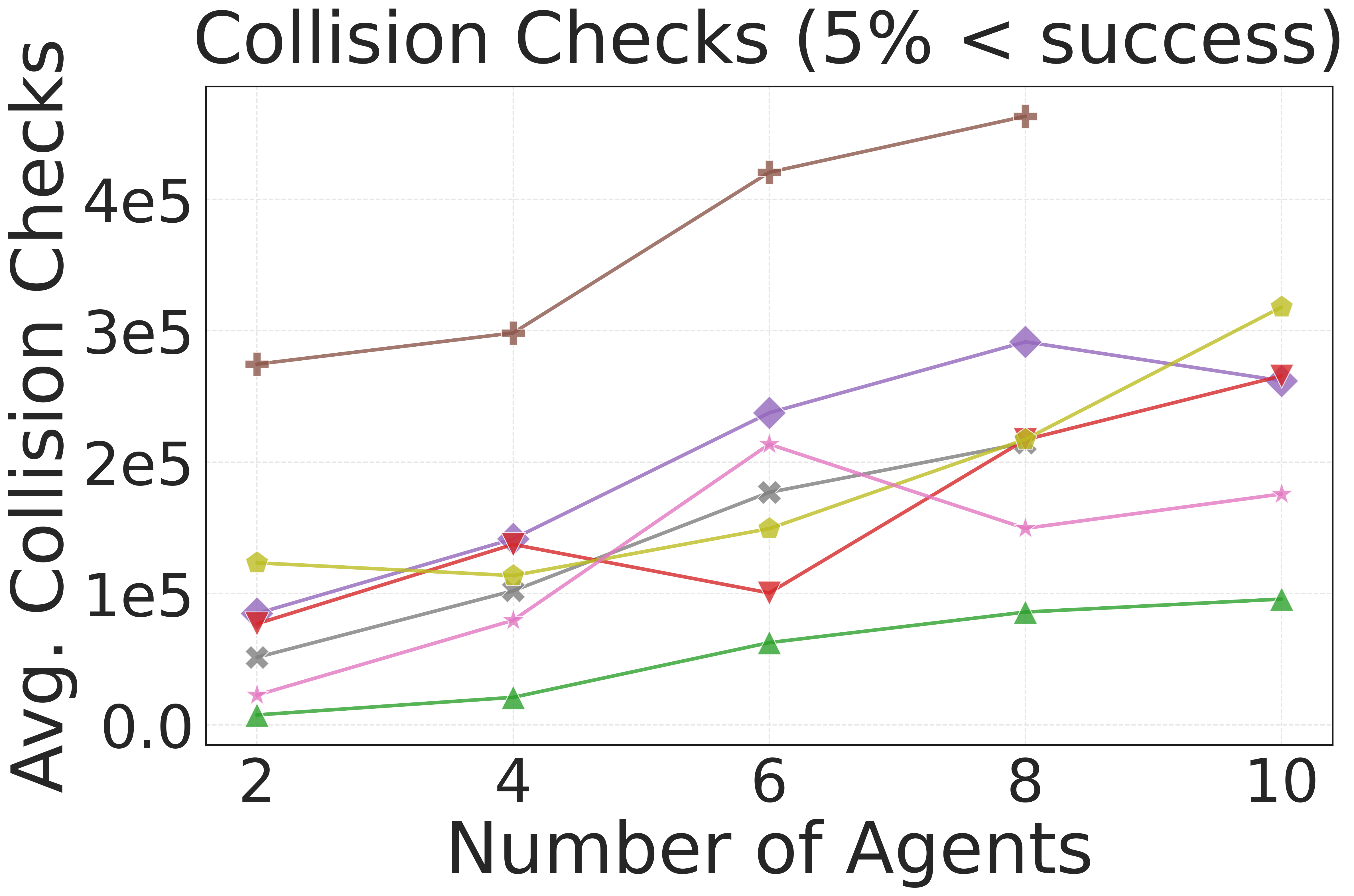}
  \end{minipage}
  }
  \caption{Scalability analysis. Top: our test scenes with 2, 4, 6, 8, and 10 robots. Bottom, from left to right: (a) success rate of methods in the $50$ planning problems of each scene. xECBS scales better than competing methods. (b) Average planning time in successful runs. xECBS maintains a lower planning time as the number of robots increases. (c) Cost per agent in successful runs. All CBS-based methods maintain similar costs while PP and sampling-based methods eventually produce worse paths. (d) Average number of collision checks. (b,c,d) report on planners with at least $5\%$ success to avoid unrepresentative data.}
  \label{fig:scalability_analysis}
\end{figure*}

\section{Experiments}
\label{sec:experiments}

To evaluate xECBS and xCBS, we created collaborative manipulation tasks with varying numbers of robots, obstacle density, and robot-robot interaction complexity, and evaluated them in simulation. We used MoveIt! \cite{coleman2014reducing} for environment handling and Isaac Sim for rendering. This setup can directly control real robots. Each robot in our experiments is a Franka Panda manipulator with 7-DoF. The experiments were conducted on an Intel Core i9-12900H with 32GB RAM (5.2GHz). 

\subsection{Experiments Setup}
Our experiments focus on testing the scalability of algorithms as well as their applicability for real-world use. We set up $7$ scenes, each with $50$ planning problems defined by starts $\compconf{\text{start}} \in \confspace{}{\text{free}}$ and goals  $\compconf{\text{goal}} \in \confspace{}{\text{free}}$.

To test the applicability of algorithms for real-world scenarios, we evaluated algorithms in two sample tasks: shelf rearrangement with 8 arms and bin-picking with 4 arms (Fig. \ref{fig:experimental_results}). For each scene, we randomly sampled 50 start and goal states from a set of task-specific configurations (e.g., pick and place configurations at different bins or positions in between shelves). Given that the robots operate within the same task space, these configurations require motion plans with substantial proximity between arms.

To shed light on how algorithms scale with the number of arms, we tested their performance in free or lightly cluttered scenes with 2, 4, 6, 8, and 10 arms as shown in Fig. \ref{fig:scalability_analysis}. 
The starts and goals for each agent are in the shared workspace region. In each setup, robots were organized in a circle, and in the cases with 6, 8, and 10 robots, a thin obstacle was placed in the circle to encourage interaction.

\subsection{Baselines}
To show the efficacy of our method, we compare it to ubiquitous algorithms commonly used to solve the M-RAMP problem and other algorithms recently applied to M-RAMP.
\subsubsection*{Sampling-Based Methods} 
We include PRM and RRT-Connect, which are arguably the most commonly used algorithms for planning in manipulation. For both, the search space is the composite state space $\confspace{}{}$. We use their OMPL \cite{OMPL2012} implementation. We include dRRT* \cite{shome2020drrt}, a more recent algorithm applied to M-RAMP that explores the composite state space via transitions on single-agent roadmaps. In our implementation, the single-agent roadmaps contain a minimum of 1,500 nodes, with increments of 1,000 added if the roadmap cannot be connected to the start or goal configurations.

\subsubsection*{Search-Based Methods} We include PP, CBS, ECBS, and CBS-MP in our comparison, as well as our proposed methods xCBS and xECBS. For all, the single agent planners are weighted A* with a uniform cost for transition and an $L_2$ joint-angle distance as a heuristic. The heuristic inflation value \footnote{Our heuristic underestimates the cost to go in radians, and edges are unit cost. The weight $w_1^L$ scales the heuristic value to match the cost of edge transitions and inflates it.} is $w_1^L = 50$, and in ECBS and xECBS, the sub-optimality bound is set to $w^H = w_2^L = 1.3$. 

Our implementation of CBS-MP differs slightly from the original in that, here, agents plan on discretized implicit graphs and not on precomputed roadmaps. This has been done to compare all search algorithms on the same planning representation.
All edge transitions on the implicit graphs are said to take one timestep.

\subsection{Evaluation Metrics and Postprocessing}
We are interested in the scalability and solution quality of algorithms. To this end, for each scene, we report the mean and standard deviation for planning time and solution cost across all segments, alongside the success rate of each algorithm in the scene. All algorithms were allocated 60 seconds for planning, after which a plan was considered a failure. The cost is the total motion (radians) carried out by the joints. 

All solutions are post-processed with a simple incremental shortcutting algorithm \cite{choset2005shortcutting}. One by one, each agent's solution path is shortcutted without creating new conflicts. Starting from the beginning of the path, the algorithm attempts to replace path segments by linear interpolations while avoiding obstacles and other agents.

\subsection{Experimental Results}
We observe that xECBS solves real-world multi-arm manipulation planning problems faster and with a higher success rate than other evaluated methods while keeping solution costs low. Both xCBS and xECBS improve on CBS and ECBS in general, however, xECBS offers a much larger boost in performance and is more suitable for M-RAMP. 

xECBS proved successful in our 4-robot bin-picking and 8-robot shelf rearrangement examples. Underscored by Fig. \ref{fig:experimental_results} (middle),
xECBS demonstrates faster planning speed (red, above $100\%$) while delivering low-cost solutions comparable to those achieved by other CBS-based approaches (blue, values below or around $100\%$).  

Our scalability analysis primarily focused on interactions between robots and minimized obstacle density. Here,
xECBS scaled to scenes with more agents better than competing methods, consistently solving more problems. The experience reuse in xECBS allowed it to query a collision checker less than the other methods we implemented and aided in reducing its planning times. 

Across our experiments, the costs of all CBS-based methods, including xCBS and xECBS, were very similar. This stability sheds light on how experience reuse can maintain solution quality while also expediting the search.
Sampling-based methods, however, faced notable challenges in scenarios with tight clutter and coordination, such as the 8-robot shelf rearrangement task. 
Not only did these methods struggle to find solutions, but the solutions they produced had high costs and variability.

\section{Conclusion}

Popular multi-agent path finding algorithms like CBS and ECBS assume fast single-agent planners, which may not be available in multi-arm manipulation tasks. To address this, we propose to accelerate conflict-based algorithms by reusing online-generated path experiences and demonstrate their benefits in xCBS and xECBS. These adaptations improve performance in multi-arm manipulation tasks while ensuring bounded sub-optimality guarantees. Our experiments demonstrate the proposed method's effectiveness in various multi-arm manipulation tasks with up to 10 arms. We observe that xECBS is particularly effective in real-world scenarios such as pick and place and shelf rearrangement, achieving higher success rates and lower planning times than currently available methods.

\section*{Acknowledgements}
The work was supported by the National Science Foundation under Grant 2328671 and the CMU Manufacturing Futures Institute, made possible by the Richard King Mellon Foundation.
\bibliography{aaai24}

\end{document}